\documentclass{article}

\usepackage[english]{babel}

\usepackage[letterpaper,top=2cm,bottom=2cm,left=3cm,right=3cm,marginparwidth=1.75cm]{geometry}

\usepackage{amsmath}
\usepackage{amssymb}
\usepackage{amsthm}
\usepackage{graphicx}
\usepackage[colorlinks=true, allcolors=blue]{hyperref}

\theoremstyle{plain}
\newtheorem{theorem}{Theorem}
\newtheorem{lemma}[theorem]{Lemma}

\theoremstyle{definition}

\title{On the clustering behavior of sliding windows}
\author{Boris Alexeev \and Wenyan Luo\thanks{Department of Mathematics, The Ohio State University, Columbus, OH} \and Dustin G.\ Mixon\footnotemark[1] \thanks{Translational Data Analytics Institute, The Ohio State University, Columbus, OH} \and Yan X Zhang\thanks{Department of Mathematics and Statistics, San Jos\'{e} State University, San Jose, CA}}
\date{}

\begin{document}
\maketitle

\begin{abstract}
Things can go spectacularly wrong when clustering timeseries data that has been preprocessed with a sliding window.
We highlight three surprising failures that emerge depending on how the window size compares with the timeseries length.
In addition to computational examples, we present theoretical explanations for each of these failure modes.
\end{abstract}

\section{Introduction}

Clustering is one of the most common tasks in data science, and given the ubiquity of timeseries data, one is naturally inclined to cluster it.
In order to perform Euclidean clustering (such as $k$-means clustering) on timeseries data, one must first map the data into Euclidean space.
This is traditionally accomplished with a \textit{sliding window}.
Explicitly, given a timeseries
\[
\left[\begin{array}{ccccc}
x(0) & x(1) & x(2) & \cdots & x(m-1)
\end{array}\right]
\]
and a nonnegative integer parameter $w$, a sliding window of length $w$ transforms the timeseries into $m-w+1$ vectors in $\mathbb{R}^w$, namely, the columns of the matrix
\[
\left[\begin{array}{ccccc}
x(0) & x(1) & x(2) & \cdots & x(m-w)\\
x(1) & x(2) & x(3) & \cdots & x(m-w+1)\\
x(2) & x(3) & x(4) & \cdots & x(m-w+2)\\
\vdots & \vdots & \vdots & & \vdots \\
x(w-1) & x(w) & x(w+1) & \cdots & x(m-1)
\end{array}\right].
\]
Naively, this seems like a reasonable approach to analyze timeseries data, but it turns out that for whatever reason, things tend to go awry with this data science pipeline.
Such problems were reported by Keogh and Lin in their provocative paper ``Clustering time-series sequences is meaningless''~\cite{KeoghL:05}.
They identified that in many cases, the clustering is \textit{meaningless} in the sense that a randomly initialized $k$-means algorithm returns a highly variable set of centroids.

Inspired by Keogh and Lin~\cite{KeoghL:05}, we performed our own investigation into clustering sliding windows of timeseries data, and we were shocked by what we observed.
In some cases, you can predict the centroid means from the \textit{multiset} of timeseries values (i.e., ``time'' doesn't matter).
In other cases, you can predict how the windows will cluster \textit{without even seeing the data}.
Which of these behaviors emerge depends on how the window size $w$ compares with the timeseries length $m$.
We also resolved the mysterious emergence of sine waves in one of the experiments presented by Keogh and Lin~\cite{KeoghL:05}.
We present our findings in the next three sections before concluding with a discussion in Section~5.

While Keogh and Lin~\cite{KeoghL:05} considered both $k$-means clustering and hierarchical clustering, for theory reasons, we focus on $k$-means clustering and spectral clustering. 
To perform spectral clustering, first run principal component analysis to find the most representative $(k-1)$-dimensional affine subspace, and then project the data onto that affine subspace before solving $k$-means for the projected data.
It turns out that spectral clustering is a $2$-approximation algorithm for $k$-means~\cite{DrineasFKVV:04}, and empirically, spectral clustering--optimal centroids are very similar to the $k$-means-optimal centroids.

Throughout, we $0$-index our timeseries, and we denote $[p]:=\{0,\ldots,p-1\}$.
Also, timeseries have length $m$, windows have size $w$, $n=m-w+1$ denotes the total number of windows, and $k$ denotes the desired number of clusters.

\section{Small windows have flat cluster centroids}

The technology company Apple went public on December 12, 1980 under the ticker symbol AAPL, and today on Kaggle, one can access the closing price of this stock for every trading day over the subsequent 42 years~\cite{Mooney:online}.
In light of the exponential growth in this stock, we take a logarithm to obtain the following timeseries of length $10590$:
\begin{center}
\includegraphics[width=0.7\textwidth,trim={2.5cm 11.5cm 2.5cm 11.5cm},clip]{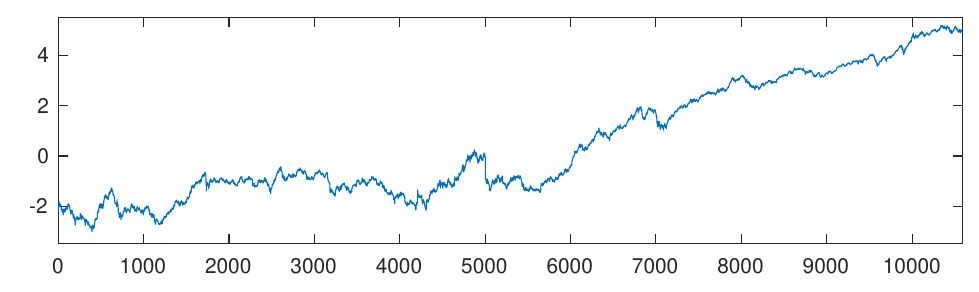}
\end{center}
Running $k$-means clustering with $k=10$ on all windows of length $100$ from this timeseries then results in the following centroids:
\begin{center}
\includegraphics[width=0.7\textwidth,trim={2.5cm 11.5cm 2.5cm 11.5cm},clip]{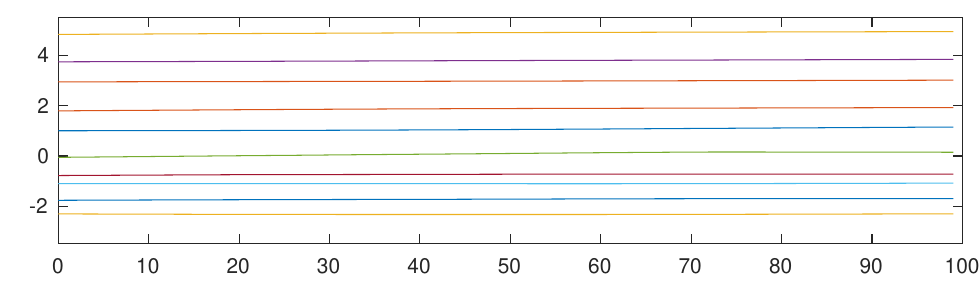}
\end{center}
Every centroid is nearly constant!
Instead of illustrating typical fluctuations in the timeseries over the course of $100$ timesteps, these centroids essentially capture how the individual timeseries values cluster together.
In this section, we show that this ``flat centroid'' phenomenon emerges as a result of the window length being ``too small'' relative to the full timeseries length.
For theory reasons, we focus on spectral clustering with only $k=2$ clusters.

As is often the case in timeseries analysis, zero padding will serve as a convenient preprocessing step in what follows.
To be explicit, the $(w-1)$-zero-padding of $x$ is defined to be the timeseries $\tilde{x}$ obtained by appending $w-1$ zeros to both sides of $x$:
\begin{center}
\footnotesize
\begin{tabular}{|c|cccc|cccc|cccc|}
\hline
$t$ & $0$ & $1$ & $\cdots$ & $w-2$ & $w-1$ & $w$ & $\cdots$ & $m+w-2$ & $m+w-1$ & $m+w$ & $\cdots$ & $m+2w-3$\\\hline
$\tilde{x}(t)$ & $0$ & $0$ & $\cdots$ & $0$ & $x(0)$ & $x(1)$ & $\cdots$ & $x(m-1)$ & $0$ & $0$ & $\cdots$ & $0$\\\hline
\end{tabular}
\normalsize
\end{center}
Notably, the left-most $w$-window of $\tilde{x}$ contains $x(0)$, while the right-most window contains $x(m-1)$.
We denote $\mathsf{pad}:=\{0,1,\ldots,w-2\}\cup\{m+w-1,m+w,\ldots,m+2w-3\}$.

\begin{theorem}
\label{thm.flat centroids}
Given a timeseries $x$ of length $m$, let $\mathcal{X}$ denote the multiset of all $w$-windows of the $(w-1)$-zero-padding of $x$.
Then the following hold:
\begin{itemize}
\item[(a)]
The centroid $\mu$ of $\mathcal{X}$ is a scalar multiple of the all-ones vector $1$.
\item[(b)]
For $k=2$, the spectral clustering--optimal centroids $\mu_1$ and $\mu_2$ of $\mathcal{X}$ both satisfy
\[
\frac{\|P_{1^\perp}\mu_i\|_2^2}{\|\mu_i-\mu\|_2^2}
\leq\frac{8w^2\|x\|_\infty^2+w^3m\|x\|_{\operatorname{Lip}}^2}{\|x-\overline{x}1\|_2^2},
\]
where $P_{1^\perp}$ is the orthogonal projection onto the orthogonal complement of the span of $1$, $\overline{x}$ denotes the average coordinate of $x$, and $\displaystyle\|x\|_{\operatorname{Lip}}:=\max_{t\in[m-1]}|x(t+1)-x(t)|$.
\end{itemize}
\end{theorem}

Part (a) above is a strengthening of Theorem~1 in~\cite{KeoghL:05} that holds in our special case.
Before proving Theorem~\ref{thm.flat centroids}, we evaluate part (b) with a couple of standard examples of $x$.
If $x$ is a random walk, then we expect $\|x\|_\infty\asymp\sqrt{m}$, $\|x\|_{\operatorname{Lip}}=1$, and $\|x-\overline{x}1\|_2\asymp m$, and so the right-hand side in (b) above is on the order of $\frac{w^3}{m}$, which is small when $w\ll m^{1/3}$.
Meanwhile, if $x=\operatorname{sin}\frac{2\pi\cdot}{p}$ for some period $p\gg1$, then $\|x\|_\infty\asymp1$, $\|x\|_{\operatorname{Lip}}\asymp\frac{1}{p}$, and $\|x-\overline{x}1\|_2\asymp\sqrt{m}$, and so our error estimate is on the order of $\frac{w^2}{m}+\frac{w^3}{p}$, which is small when $w\ll\min\{m^{1/2},p^{1/3}\}$.

\begin{proof}[Proof of Theorem~\ref{thm.flat centroids}]
For (a), since $\mathcal{X}$ consists of all $w$-windows of the $(w-1)$-zero-padding of $x$, each coordinate of $\mu$ is the average of the $m$ coordinates of $x$ along with $w-1$ zeros.
For (b), note that $\mu_1$ and $\mu_2$ both reside in the $1$-dimensional affine subspace delivered by principal component analysis.
Accordingly, let $X$ denote the matrix of column vectors from $\mathcal{X}$, center these column vectors by taking $B:=X-\mu1^\top$, and let $u$ denote the top left-singular vector of $B$.
Since there necessarily exist $c_1,c_2\in\mathbb{R}$ such that $\mu_i=\mu+c_i u$ for each $i$, it follows that
\[
\frac{\|P_{1^\perp}\mu_i\|_2^2}{\|\mu_i-\mu\|_2^2}
=\|P_{1^\perp}u\|_2^2
=1-\langle u,\tfrac{1}{\sqrt{w}}1\rangle^2
\leq\frac{\|B-\frac{1}{w}11^\top B\|_F^2}{\sigma_{1}(B)^2}
=\frac{\|P_{1^\perp}X\|_F^2}{\sigma_{1}(B)^2},
\]
where the inequality follows from Wedin's $\sin\Theta$ theorem~\cite{Wedin:72}.
We estimate the denominator by
\[
\sigma_{1}(B)^2
\geq\|e_1^\top B\|_2^2
\geq\|x-\overline\mu1\|_2^2
\geq\|x-\overline{x}1\|_2^2.
\]
Next, we estimate the numerator in terms of the $(w-1)$-zero-padding $\tilde{x}$ of $x$:
\[
\|P_{1^\perp}X\|_F^2
=\sum_{s=0}^{m+w-2} \sum_{t=0}^{w-1}\bigg(\tilde{x}(s+t)-\frac{1}{w}\sum_{t'=0}^{w-1}\tilde{x}(s+t')\bigg)^2
\leq\sum_{s=0}^{m+w-2} \sum_{t=0}^{w-1}\bigg(\frac{1}{w}\sum_{t'=0}^{w-1}\Big|\tilde{x}(s+t)-\tilde{x}(s+t')\Big|\bigg)^2.
\]
Considering
\[
\Big|\tilde{x}(s+t)-\tilde{x}(s+t')\Big|
\leq\left\{\begin{array}{cl}
2\|x\|_\infty&\text{if }s\in\mathsf{pad} \text{ or } s+w-1\in\mathsf{pad}\\
(w-1)\|x\|_{\operatorname{Lip}} &\text{else},
\end{array}\right.
\]
the result follows.
\end{proof}

\section{Clustering nearly symmetric data is meaningless}

In his PhD thesis~\cite{Saito:94}, Naoki Saito generated a synthetic dataset to evaluate various classification algorithms.
This dataset, which is available online~\cite{Saito:online}, consists of independent draws from three specially crafted distributions of random timeseries of length $128$, and each distribution is named after the shape of the resulting timeseries: \textit{cylinder}, \textit{bell}, and \textit{funnel}; see Example~4.7 in~\cite{Saito:94} for details.
Running $k$-means on the training set (which consists of $30$ examples) results in the following centroids:
\begin{center}
\includegraphics[width=0.7\textwidth,trim={2.5cm 11.5cm 2.5cm 11.5cm},clip]{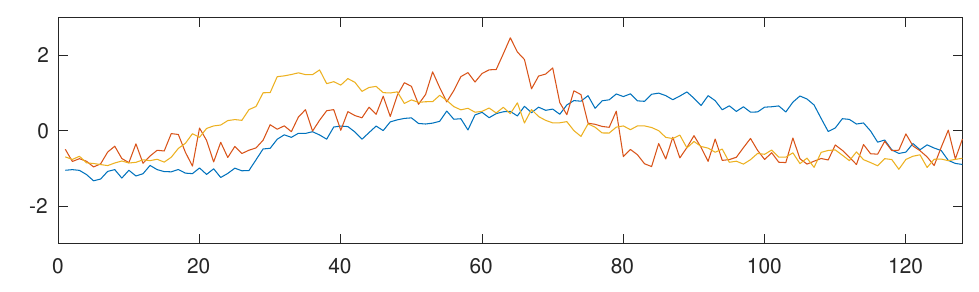}
\end{center}
Let's use the cylinder--bell--funnel dataset to replicate an experiment described in Section~4 of~\cite{KeoghL:05}.
Concatenate the examples in the training set to form a timeseries of length $30\cdot128=3840$, and then run $k$-means clustering with $k=3$ on all windows of length $128$ from this timeseries.
The resulting centroids are depicted below:
\begin{center}
\includegraphics[width=0.7\textwidth,trim={2.5cm 11.5cm 2.5cm 11.5cm},clip]{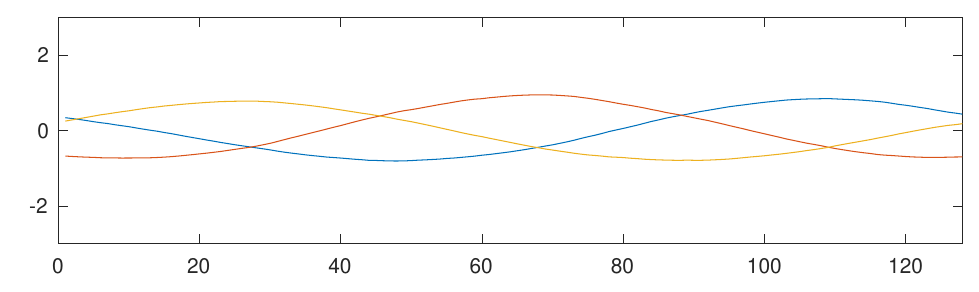}
\end{center}
Every centroid is nearly a sine wave!
This comes as a surprise since none of the examples in the training set exhibit such regularity.
However, the discrete Fourier transform of our $3840$-long timeseries offers a hint for what's going on here:
The dominant sinusoidal contribution has period $128$ (i.e., $w$), with all others being much smaller by comparison.
In this section, we show that (nearly) sinusoidal centroids emerge in such settings.

The following result is written in terms of a certain notion of distance between subspaces.
The \textit{chordal distance} between two $r$-dimensional subspaces of $\mathbb{R}^d$ is given by
\[
\sqrt{\sin^2\theta_1+\cdots+\sin^2\theta_r},
\]
where $\theta_1,\ldots,\theta_r$ denote the principal angles between these subspaces.
This defines a metric on the Grassmannian space $\operatorname{Gr}(r,\mathbb{R}^d)$.
Note that this distance is small only if all of the principal angles are small, meaning every unit vector in one of the subspaces is close to the other subspace.

\begin{theorem}
\label{thm.sine wave}
Given positive integers $w$ and $p$, consider the timeseries $x$ of length $m:=pw+w-1$ defined by
\[
x(t)
=a + b\sin\frac{2\pi(t-c)}{w}+e(t)
\]
for some $a,b,c\in\mathbb{R}$ and some timeseries $e$ with mean zero.
Let $\mathcal{X}$ denote the multiset of all $w$-windows of $x$.
Then the following hold:
\begin{itemize}
\item[(a)]
The centroid $\mu$ of $\mathcal{X}$ satisfies $\|\mu\|_{\operatorname{Lip}}\leq\frac{2}{pw}\|e\|_\infty$.
\item[(b)]
For $k=3$, the spectral clustering--optimal centroids $\mu_1$, $\mu_2$, and $\mu_3$ of $\mathcal{X}$ all have the property that $\mu_i-\mu$ resides in a $2$-dimensional subspace whose chordal distance from the span of $\cos\frac{2\pi\cdot}{w}$ and $\sin\frac{2\pi\cdot}{w}$ is at most
\[
\Bigg(\frac{|b|}{\sqrt{\frac{4}{pw}\|e\|_2^2+\frac{12}{p}\|e\|_\infty^2}}-1\Bigg)^{-1}
\]
whenever this quantity is positive.
\end{itemize}
\end{theorem}

Part (a) above is yet another version of Theorem~1 in~\cite{KeoghL:05} that holds in our special case.
Note that the bound in part (b) can be interpreted as $(\mathsf{snr}-1)^{-1}$, where $\mathsf{snr}$ is one way to quantify the \textit{signal-to-noise ratio} in our setup.
By Theorem~\ref{thm.sine wave}, when $\mathsf{snr}$ is larger, we can expect the spectral clustering--optimal centroids to be more sinusoidal.

Before proving Theorem~\ref{thm.sine wave}, we highlight the significance of our choice $m:=pw+w-1$.
Suppose $p=2$, $w=3$, and $e(t)=0$ for all $t$.
Then $m=8$, and $x$ takes the form
\[
\left[\begin{array}{cccccccc}
\alpha & \beta & \gamma & 
\alpha & \beta & \gamma & 
\alpha & \beta 
\end{array}\right].
\]
The length-$w$ windows of $x$ are then given by the columns of the matrix
\[
\left[\begin{array}{ccc|ccc}
\alpha & \beta & \gamma & \alpha & \beta & \gamma \\ 
\beta & \gamma & \alpha & \beta & \gamma & \alpha \\ 
\gamma & \alpha & \beta & \gamma & \alpha & \beta
\end{array}\right].
\]
In general, if $x$ is $w$-periodic with length $m$, then we end up with $pw$ windows that can be partitioned into $p$ equal batches of size $w$, with each batch consisting of all cyclic permutations of the first window.
We will exploit this structure to demonstrate the emergence of sinusoidal centroids.

\begin{proof}[Proof of Theorem~\ref{thm.sine wave}]
For (a), we have
\[
\mu(t)
=a-\frac{1}{pw}\sum_s e(s),
\]
where the sum is over $s\in\{0,\ldots,t-1\}\cup\{m-w+t+1,\ldots,m-1\}$.
(Here and throughout, we make use of identities like $\sum_{t=0}^{w-1}\sin\frac{2\pi(t-c)}{w}=0$, which can be deduced by interpreting the left-hand side as the imaginary part of a vanishing geometric sum.)
Thus,
\[
|\mu(t+1)-\mu(t)|
=\frac{1}{pw}|e(m-w+t+1)-e(t)|
\leq\frac{2}{pw}\|e\|_\infty.
\]
For (b), we use the fact that each $\mu_i$ resides in the $2$-dimensional affine subspace given by principal component analysis.
Consider the matrix $X$ whose column vectors belong to $\mathcal{X}$, and center these vectors to get $B:=X-\mu1^\top$.
It will be useful to consider the matrices $A$ and $E$ defined by
\[
A_{s,t}=b\sin\frac{2\pi(s+t-c)}{w},
\qquad
E_{s,t}=a+e(s+t)-\mu(s),
\]
since this allows for the decomposition $B=A+E$, where $A$ is ``structured'' and $E$ is ``small''.
In particular, $AA^\top$ is a circulant matrix:
\[
(AA^\top)_{s,s'}
=\sum_{t=0}^{pw-1}b^2\sin\frac{2\pi(s+t-c)}{w}\sin\frac{2\pi(s'+t-c)}{w}
=\frac{b^2pw}{2}\cos\frac{2\pi(s-s')}{w}.
\]
Thus, the eigenvalues of $AA^\top$ are given by the discrete Fourier transform of $\frac{b^2pw}{2}\cos\frac{2\pi\cdot}{w}$, from which one may conclude that $AA^\top$ is $\frac{b^2pw^2}{4}$ times the orthogonal projection onto the span of $\cos\frac{2\pi\cdot}{w}$ and $\sin\frac{2\pi\cdot}{w}$.
Next, we show that $E$ is small:
\[
|E_{s,t}|
=|a+e(s+t)-\mu(s)|
\leq|e(s+t)|+|a-\mu(s)|
\leq|e(s+t)|+\frac{1}{p}\|e\|_\infty,
\]
which implies
\[
|E_{s,t}|^2
\leq|e(s+t)|^2+2\cdot|e(s+t)|\cdot\frac{1}{p}\|e\|_\infty+\frac{1}{p^2}\|e\|_\infty^2
\leq|e(s+t)|^2+\frac{3}{p}\|e\|_\infty^2,
\]
and so
\[
\|E\|_F^2
\leq w\|e\|_2^2+3w^2\|e\|_\infty^2.
\]
Thus, Wedin's $\sin\Theta$ theorem~\cite{Wedin:72} and Weyl's inequality for singular values gives that the desired chordal distance is at most
\[
\frac{\|B-A\|_F}{\sigma_2(B)}
\leq\frac{\|E\|_F}{\sigma_2(A)-\|E\|_F}
=\bigg(\frac{\sigma_2(A)}{\|E\|_F}-1\bigg)^{-1}
\leq\Bigg(\frac{|b|}{\sqrt{\frac{4}{pw}\|e\|_2^2+\frac{12}{p}\|e\|_\infty^2}}-1\Bigg)^{-1}.
\qedhere
\]
\end{proof}

Keogh and Lin~\cite{KeoghL:05} used the cylinder--bell--funnel example to illustrate their observation that clustering sliding windows is meaningless.
In retrospect, the primary reason for meaninglessness in this case is the underlying symmetry:
These sliding windows can be viewed as noisy versions of a single orbit under the action of cyclic permutations.
Meanwhile, if we were to cluster the orbit directly, any centroids that emerge would have the same $k$-means value as their cyclic permutations.
This is an instance of a more general phenomenon: 
\begin{center}
\textit{Clustering nearly symmetric data is meaningless.}
\end{center}
For example, run $k$-means clustering on thousands of independent realizations of a spherical Gaussian random vector in $\mathbb{R}^2$.
Since the underlying distribution is rotation invariant, we can expect any rotation of the $k$-means-optimal centroids to perform similarly well as cluster centroids.

\section{Large windows have interval clusters}

Empirically, we observe that when the window is a sizable fraction of the timeseries, the optimal clusters are frequently intervals.
As an example, we performed $1000$ trials of the following experiment:
Draw a random walk of length $m=100$, and then for each $w\in\{1,2,\ldots,m-k+1\}$, run $k$-means with $k=3$ on all $w$-windows of the random walk, and record whether all of the resulting clusters are intervals.
Here's a plot of the proportion of trials that produced interval clusters for each $w$:
\begin{center}
\includegraphics[width=0.7\textwidth,trim={2.5cm 11.5cm 2.5cm 11.5cm},clip]{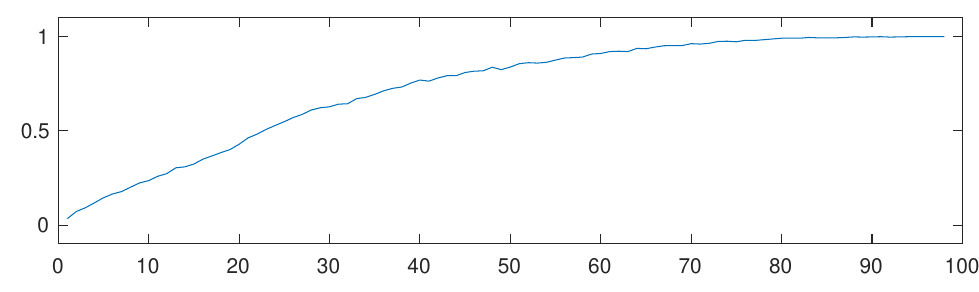}
\end{center}
In this section, we provide some theoretical explanation for this phenomenon.

Consider independent random variables $Z_1,\ldots,Z_{m-1}$ with zero mean and unit variance, and define a random length-$m$ timeseries $x$ by
\[
x(t)
=\sum_{i=1}^{t}Z_i.
\]
(Notably, $x(0)=0$.)
Take $n:=m-w+1$, and for each $s\in[n]$, let $x_s$ denote the window
\[
\left[\begin{array}{cccc}
x(s) & x(s+1) & \cdots & x(s+w-1)
\end{array}\right].
\]
We are interested in the partition $C_1\sqcup\cdots \sqcup C_k=[n]$ that minimizes the $k$-means objective:
\[
\sum_{i=1}^k\frac{1}{|C_i|}\sum_{s,s'\in C_i}\|x_s-x_{s'}\|^2.
\]
When $w$ is large, we expect the $k$-means objective of each partition to concentrate at its expectation (especially when the $Z_i$'s have subgaussian distribution), and so the quantity to minimize becomes a slight perturbation of
\[
\sum_{i=1}^k\frac{1}{|C_i|}\sum_{s,s'\in C_i}\mathbb{E}\|x_s-x_{s'}\|^2.
\]
This is the setting of our main result in this section:

\begin{theorem}
\label{thm.random walk}
For the random $\{x_s\}_{s\in[n]}$ defined above, the partitions $C_1\sqcup\cdots \sqcup C_k=[n]$ that minimize
\[
\sum_{i=1}^k\frac{1}{|C_i|}\sum_{s,s'\in C_i}\mathbb{E}\|x_s-x_{s'}\|^2
\]
are precisely those for which each $C_i$ is an interval of length $\lfloor\frac{n}{k}\rfloor$ or $\lceil\frac{n}{k}\rceil$.
\end{theorem}

Before proving this theorem, we present a couple of helpful lemmas.

\begin{lemma}
\label{lem.var is dist}
$\mathbb{E}\|x_s-x_{s'}\|^2=w|s-s'|$.
\end{lemma}

\begin{proof}
Without loss of generality, we have $s'<s$, in which case
\begin{align*}
\mathbb{E}\|x_s-x_{s'}\|^2
=\mathbb{E}\sum_{t=0}^{w-1}\Big(x(s+t)-x(s'+t)\Big)^2
&=\mathbb{E}\sum_{t=0}^{w-1}\bigg(\sum_{i=s'+t+1}^{s+t}Z_i\bigg)^2\\
&=\sum_{t=0}^{w-1}\sum_{i=s'+t+1}^{s+t}\sum_{j=s'+t+1}^{s+t}\mathbb{E}Z_iZ_j
=w(s-s'),
\end{align*}
and the result follows.
\end{proof}

\begin{lemma}
\label{lem.intervals are best}
For every subset $C\subseteq\mathbb{Z}$ of size $r$, it holds that
\[
\sum_{s,s'\in C}|s-s'|
\geq \frac{r(r+1)(r-1)}{3},
\]
with equality precisely when $C$ is an interval.
\end{lemma}

\begin{proof}
Take $C=\{s_0,\ldots,s_{r-1}\}$ with $s_0<\cdots<s_{r-1}$.
Then $|s_i-s_j|\geq |i-j|$ for all $i,j\in[r]$, and so
\[
\sum_{s,s'\in C}|s-s'|
=\sum_{i,j\in[r]}|s_i-s_j|
\geq\sum_{i,j\in[r]}|i-j|,
\]
with equality precisely when $C$ is an interval.
Standard summation identities simplify this lower bound:
\[
\sum_{i,j\in[r]}|i-j|
=2\sum_{i=1}^{r-1}\sum_{j=1}^i j
=\sum_{i=1}^{r-1}(i^2+i)
=\frac{r(r-1)(2r-1)}{6}+\frac{r(r-1)}{2}
=\frac{r(r+1)(r-1)}{3}.
\qedhere
\]
\end{proof}

\begin{proof}[Proof of Theorem~\ref{thm.random walk}]
By Lemma~\ref{lem.var is dist}, we seek to minimize
\[
\sum_{i=1}^k\frac{1}{|C_i|}\sum_{s,s'\in C_i}|s-s'|,
\]
and Lemma~\ref{lem.intervals are best} implies that the minimizers have the property that each $C_i$ is an interval.
It remains to establish that for each $i$, the size $r_i$ of $C_i$ is $\lfloor\frac{n}{k}\rfloor$ or $\lceil\frac{n}{k}\rceil$.
By Lemma~\ref{lem.intervals are best}, the quantity to minimize is
\[
\sum_{i=1}^k\frac{1}{r_i}\cdot\frac{r_i(r_i+1)(r_i-1)}{3}
=\frac{1}{3}\bigg(\sum_{i=1}^k r_i^2-k\bigg)
\]
subject to the constraints that each $r_i$ is a positive integer and $r_1+\cdots+r_k=n$.

If $r_i>r_j+1$, then it's better to take $r_i\leftarrow r_i-1$ and $r_j\leftarrow r_j+1$:
\[
(r_i-1)^2+(r_j+1)^2
=r_i^2+r_j^2-2(r_i-r_j-1)
<r_i^2+r_j^2.
\]
As such, $(r_1,\ldots,r_k)$ is suboptimal if any two coordinates differ by more than $1$.
The only remaining options take each $r_i$ to be $\lfloor\frac{n}{k}\rfloor$ or $\lceil\frac{n}{k}\rceil$.
Furthermore, the constraint $r_1+\cdots+r_k=n$ forces $n\bmod k$ of the coordinates to equal $\lceil\frac{n}{k}\rceil$ and the other coordinates to equal $\lfloor\frac{n}{k}\rfloor$.
All of these remaining choices have the same value, so they all must be optimal.
\end{proof}

\section{Discussion}

In this paper, we highlighted a few phenomena that arise when clustering sliding windows.
We now point out two observations that warrant further investigation.

First, Theorems~\ref{thm.flat centroids} and~\ref{thm.sine wave} are both restrictive in how to select $k$.
In the setting of Theorem~\ref{thm.flat centroids}, we empirically observe flat cluster centroids when $k>2$, and in the setting of Theorem~\ref{thm.sine wave}, we similarly observe sinusoidal clusters when $k>3$.
However, to prove both theorems, we argue by relating the covariance matrix of the data to another matrix of some rank $r$, and so we can only prove guarantees for spectral clustering with $k\leq r+1$.
Of course, if the rank of the covariance matrix is \textit{exactly} $r$, then the optimal clustering centroids necessarily reside in the $r$-dimensional affine subspace implicated by principal component analysis, even if $k>r+1$.
Is there a way to estimate the proximity of $k>r+1$ optimal clustering centroids to this affine subspace in general?

Second, we pointed out that clustering nearly symmetric data is meaningless, but it would be interesting to determine the extent to which this explains the meaninglessness observed by Keogh and Lin~\cite{KeoghL:05}.
In particular, is it always the case that when sliding windows exhibit meaningless clustering, there exists an approximate symmetry that acts on the windows?
And is this approximate symmetry necessarily given by cyclic permutations?

\section*{Acknowledgments}

DGM was supported in part by NSF DMS 2220304.

\end{document}